%% file: main.tex
\documentclass{svproc}

\usepackage{type1cm}        
\usepackage{makeidx}         
\usepackage{graphicx}        
\usepackage{multicol}        
\usepackage{subcaption}
\usepackage{multirow}

\usepackage[varvw]{newtxmath}       

\newcommand{\real}{\mathbb{R}}

\newcommand{\GG}{\mathcal{G}}
\newcommand{\VV}{\mathcal{V}}
\newcommand{\EE}{\mathcal{E}}
\newcommand{\NN}{\mathcal{N}}

\newcommand{\LL}{\mathcal{L}}
\newcommand{\diag}{{\mathrm{diag}}}

\newtheorem{assumption}{Assumption}

\usepackage{todonotes}

\begin{document}
\mainmatter              
\title{
Proportional Control for Stochastic Regulation on
Allocation of Multi-Robots
}
\titlerunning{Control for Stochastic Regulation on Allocation of
Multi-Robots}

\author{Thales C. Silva \and Victoria Edwards \and M. Ani Hsieh
 }
%
\authorrunning{Thales C. Silva et al.} 
%
\tocauthor{Thales C. Silva, Victoria Edwards, and M. Ani Hsieh
 }

\institute{Thales C. Silva 
\at 
Department of Mechanical Engineering and Applied Mechanics,
University of Pennsylvania, Philadelphia,
USA.
\email{scthales@seas.upenn.edu}
}

\institute{University of Pennsylvania,
Department of Mechanical Engineering and Applied Mechanics, 
Philadelphia, PA 19104, USA\\
\email{scthales@seas.upenn.edu}
}

\maketitle 

%

\input{sections/0_abstract}

\section{Introduction}
\label{intro}
\input{sections/1_introduction.tex}

\section{Problem formulation}
\label{ProbForm}
\input{sections/2_section.tex}

\section{Methodology}
\label{section3}
\input{sections/3_section.tex}
   
\section{Results}
\label{simulation}
\input{sections/5_simulation.tex}

\section{Discussion}
\label{conclusion}
\input{sections/discussion.tex}

\section{Conclusions and future work}
\input{sections/6_conclusion.tex}
	
\section*{Acknowledgements}
We gratefully acknowledge the support of ARL DCIST CRA W911NF-17-2-0181,
Office of Naval Research (ONR)
Award No. N00014-22-1-2157,
and the National Defense Science \& Engineering Graduate (NDSEG) Fellowship Program.

\bibliographystyle{spmpsci}
\bibliography{main.bbl}
\end{document}

%% file: sections/0_abstract.tex
\begin{abstract}
Any strategy used to distribute a robot ensemble over a set of sequential tasks is subject
to inaccuracy due to robot-level
uncertainties and environmental influences
on the robots' behavior. 
We approach the problem of inaccuracy during
task allocation by modeling and controlling the overall
ensemble behavior. 
Our model represents the allocation problem as a stochastic
jump process and we regulate the mean and variance of such a process. 
The main contributions of this paper are:
Establishing a structure for the transition rates 
of the equivalent stochastic jump process and formally showing 
that this approach 
leads to decoupled parameters that allow us to adjust the
first- and second-order moments of
the ensemble distribution over tasks,
which gives the
flexibility to decrease the variance
in
the desired final distribution.
This allows us to directly shape the impact of uncertainties on the
group
allocation over tasks.
We introduce a detailed procedure to design the gains to achieve the desired mean and show how the additional parameters impact the
covariance matrix, which is directly associated with 
the degree of task allocation precision.
Our simulation and experimental results illustrate the
successful control of several
robot ensembles
during task allocation.
\keywords{multi-robot system, control, task-allocation}
\end{abstract}

%% file: sections/1_introduction.tex
Modeling an ensemble of robots
as an aggregate dynamical system
offers flexibility in the
task assignment and is an alternative to
traditional bottom-up task allocation methods, which 
usually are computationally expensive 
(see \cite{Elamvazhuthi_2019,gerkey2004} and references
therein). Given a desired allocation of robots to tasks, each robot must navigate, handle dynamic constraints, and interact with the environment to achieve the desired allocation while meeting some desired team-level performance specifications.  It is well understood that uncertainties resulting from the
robot's interactions with the environment, execution of its assigned tasks, and noise from its own sensors and actuators might lead to
several issues 
during task allocation  
({\it e.g.,} performance loss and inaccuracy)
\cite{Nam15,Nam17}.  This is because allocations are often computed prior to execution and do not account for uncertainties that arise during runtime. 
In addition, analysis of the team-level performance of swarms has shown that local 
(and sometimes small) deviations from
the required performance
at the robot level can combine and
propagate,
leading to substantial changes in the behavior 
of the whole group, which might result in loss of
performance for the entire team (see Section 2 in \cite{Hamann2012}).
Consequently, addressing the fundamental question of
how to properly represent and design robot ensemble behaviors
to meet the desired performance requirements can help us to improve our
understanding of such 
complex systems.

An existing category of ensemble models uses stochastic processes to represent 
robot teams 
\cite{Elamvazhuthi2021,hsieh2008biologically,Bill2011}.
These models use random variables to describe the
population of robots performing tasks, whereas transitions between tasks are described
based on jump processes with random arrival times.
Using these representations
it is possible to
design stable equilibria that reflect desired distributions
of the robots over tasks \cite{berman2009,Deshmukh2018}.
In addition, it is also possible
to account for and to incorporate heterogeneous sensing and mobility capabilities in a robot team using these models
\cite{prorok2017impact,ravichandar2020}.
One of the major advantages of these macroscopic representations is their ability to scale to larger team sizes since their complexity is solely a function of the number of tasks.  In contrast, the complexity of traditional task allocation
methods grows as the number of agents and tasks increases.

Nevertheless, most macroscopic approaches 
are accurate for 
an asymptotically large number 
of agents,
{\it i.e.,} when $N \rightarrow \infty$, where $N$ represents the number
of robots in the team.  Fundamentally, within this assumption is the notion that 
no individual robot's deviations from its desired performance will greatly impact the team's performance as a whole.
And yet Tarapore {\it et al.,} \cite{Tarapore2020} recently discussed the challenges associated with employing large teams of robots in real world applications. Considering the impracticability of dense swarms in several applications,
they propose
{\it sparse swarms} 
arguing that guiding the research toward low-density
swarms is more relevant to real-world applications.
One of the goals of this paper is to address the large
number assumption by providing a
method to regulate the variance of a robot distribution in the task allocation problem. 
Specifically, by explicitly controlling the mean of the robot distribution
and independently controlling its variance
we can employ our methodology
in teams with a relatively small number of
agents, that is, we decrease the impact
of individual robot deviations on the
team's performance.
Hence, our methodology is a stride
towards
accurately modeling the allocation dynamics of these {\it sparse swarms}. While similar techniques exist \cite{Bill2012}, the proposed strategy lacks the ability to simultaneously control both the first- and second-order moments. Instead, we leverage
the work from Klavins 
\cite{klavins2010proportional}, which considers a 
{\it stochastic hybrid system} to model and control the mean and variance of molecular species concentrations.
We extend this result to robot swarm applications,
which fundamentally alter some of the assumptions in 
\cite{klavins2010proportional}, for example, in our setting
we are concerned about
teams of constant size without a set of
continuous variables.
In robotics, Napp {\it et al.,} \cite{Napp2011}
presented an integral
set point regulation technique, 
employed to regulate the mean of the robot population performing specific 
tasks. While they
highlighted 
the importance of considering variance
control,
since the
mean can be a
sensitive control objective, 
they 
{do not} propose
a methodology to adjust
the variance.

In this paper, we present a strategy to govern the 
mean and variance of the distribution of robots across a set of tasks.
The contributions of this work include
1) a decoupled formulation of the first- and 
second-order moments for an ensemble model of robot teams
that allows their adjustment individually, and 
2) systematic methods to determine the control gains
to achieve the desired distributions of robots over tasks. 
Our experimental results indicate 
that our proposed macroscopic techniques allows for improved control of small number teams.







%% file: sections/2_section.tex

\subsection{Graph Theory}
\label{sec:graph_theory}
We are interested in the
stochastic behavior 
of a robot ensemble
generated by robot-level uncertainties.
We assume a group of $N\in\mathbb{N}$ robots executing $M\in\mathbb{N}$ 
spatially distributed tasks,
the relationship among tasks is represented by a 
graph $\GG(\VV,\EE)$,
in which the elements of the node set $\VV=\{1,...,M\}$ 
represent the tasks, and the edge set
$\EE\subset \VV \times \VV$ 
maps 
adjacent allowable switching between task
locations.
Each element of the edge set represents a
directed path from $i$ to $j$ 
if $(i,j)\in\EE$.
A graph $\mathcal{G}$ is called an
{\it undirected graph} if $(i,j)\in \mathcal{E}
\iff (j,i)\in \mathcal{E}$.
We define the neighborhood of the $i$th {task}
node as the set formed by all {tasks} that have the
$i$th {task} as a child node and it is denoted by
$\NN_i=\{j\in\VV : (j,i)\in\EE\}$.
We call
a neighbor of a task $i$ an adjacent task $j$
that belongs to $\NN_i$.
In this paper we assume the following:
\begin{assumption}
\label{assump:undirected}
The graph that maps the relationship
between tasks, $\GG(\VV,\EE)$, is connected and undirected.
\end{assumption}
Assumption \ref{assump:undirected} allows 
the robots to sequentially move back and forth
between tasks.
For some task topologies, it is possible to adjust the
transition parameters
such that directionality in the task graph is obtained
\cite{Deshmukh2018}.
\subsection{Ensemble Model}
\label{sec:SI_model}
To derive the macroscopic ensemble model for the $N$ robot team simultaneously executing $M$ tasks,
we assume each robot can execute only one task at a given time
and
it must switch between
adjacent tasks to complete the team's  wide objective.
The model needs to capture the 
global stochastic characteristics that emerge from
a robot's interaction with other robots and the environment, such as variations in terrain \cite{prorok2017impact,Bill2012},
collision avoidance \cite{Bill2011}, and
delays \cite{berman2009}.
Thus, the state of our system is given by a vector 
of random variables
$\boldsymbol{X}(t)=[X_1(t)~\cdots~X_M(t)]'$,
where the prime symbol in
a vector, $\boldsymbol{a}'$, denotes its
transpose,
and each element of $\boldsymbol{X}(t)$ describes
the number of robots executing the respective task at time $t$, 
therefore $\sum_{i=1}^MX_i(t)=N$ for any $t\geq 0$.

We aim to characterize the evolution of the random
variable $X_i(t)$, for each $i\in\{1,\dots,M\}$ and,
ultimately, regulate
its value.
However, the stochastic nature of these
variables might lead to convergence only to a vicinity
of the desired value, where the size of the convergence
region depends on robot-level  
deviations from the desired performance
\cite{KHALIL2002ab}.
In principle, 
the control of systems subject to random variations
could be approached by applying robust control techniques
and defining input-to-state stability criteria. 
However, in our scenario, defining bounds for variations of 
robot-level behavior for a general number of tasks and
robots can be difficult.
Therefore, 
we approach the problem by controlling the mean value of the
state variables $X_i(t)$ and  
the dynamics of their second-order
moments.
Observe that the vicinity size of the
desired objective is directly associated
with the second-order moment.

The states of our system
characterize pure
jump processes $X_i(t):[0,\infty)\rightarrow \mathcal{D}$, where
$\mathcal{D}$ is a finite discrete set.
In particular, let $\psi:\mathcal{D}\times[0,\infty)\rightarrow \real$
be a real-valued function, it is a known result
that the dynamics
of the expected value of $\psi$ is described by
the following relation
(see Chapter 1 in \cite{Oksendal2007}),
\begin{align*}
   \frac{d}{dt}E[\psi(X_i(t))] ={} &
   E[\LL\psi(X_i(t))],
\end{align*}
in which $\LL$ is the infinitesimal generator of the
stochastic process defined as 
\begin{align}
    \label{eq:inf_generator}
   \LL\psi(X_i(t)) ={} & 
   \sum_{\ell}[\psi(\phi_\ell(X_i(t))) - 
   \psi(X_i(t))
    ]\lambda_\ell(X_i(t)),
\end{align}
where  $\lambda_\ell(X_i(t))$ 
is a function of the number of
robots at task $X_i$ and in its neighborhood which describes
the rate of transitions, {\it i.e.}, $\lambda_\ell(X_i(t))dt$ represents
the probability of a transition in $X_i(t)$ occurring in
$dt$ time, $\phi_\ell(X_i(t))$ maps the size of a change
on $X_i(t)$ given that an $\ell$
transition happened. 
In our case, we assume two
possible unitary changes 
in $dt$,
a transition 
of type $\ell=1$, in which 
one agent leaves $X_i(t)$
defined by
$\phi_1(X_i(t))=X_i(t)-1$, 
and of type $\ell =2$ where
one agent arrives $X_i(t)$, defined by
$\phi_2(X_i(t))=X_i(t)+1$.
Hence, $\ell\in\{1,2\}$ maps the type
of transition, {\it i.e.,} one
robot leaving or one agent arriving
at $X_i(t)$, respectively.

In this paper, we are interested in enforcing
desired first-
and second-order moments of the  distribution of robots across 
tasks.
Specifically,
the problem investigated can be stated as:

\bigskip
\noindent \textbf{Problem 1.} 
Given $M$ tasks to be dynamically executed by $N$ robots,
in which the switching between adjacent tasks is 
represented by the
topology of the task graph $\GG(\VV,\EE)$,
define a feedback controller that changes the transition rates $\lambda_{\ell}(X_i)$,
for all $i\in\{1,\dots,M\}$,
such that
the distribution of first- and second-order moments 
converge to
the desired value.

\bigskip
\noindent It is important to note that a solution for a similar
problem was
previously proposed in \cite{Bill2011,Bill2012}.
Even though they
were motivated by the study from Klavins \cite{klavins2010proportional},
the proposed control strategy relies on a feedback 
linearization action,
which leads to a closed-loop system of the form 
(see equation $(9)$ in \cite{Bill2011})
\begin{align*}
    \frac{d}{dt}E[\boldsymbol{X}(t)]=
    (\boldsymbol{K}_{\alpha}+\boldsymbol{K}_{\beta})
    E[\boldsymbol{X}(t)],
\end{align*}
where $\boldsymbol{K}_{\alpha}$ is a matrix 
that maps the transition rate
of the system and $\boldsymbol{K}_{\beta}$ is a matrix with new parameters
to attain the desired variance.
The main drawback of {this} approach is twofold: 
i) by changing $\boldsymbol{K}_{\beta}$ the stationary distribution of
the equivalent
Markov chain also changes, which makes it necessary to
readjust the gains for the
whole system, and ii) addressing the regulation of
second-order moments using $\boldsymbol{K}_{\beta}$ is
{\it equivalent} to changing the transition rate of the whole process
during the entire evolution.
In comparison, our proposed strategy overcomes these limitations by simply manipulating
the graph structure and defining a new structure for 
$\lambda_\ell(\cdot)$,
in such a way that the new feedback strategy
does not directly modify the stationary distribution of the
equivalent Markov
chain and provides changes to the transition rates according to the
current state of the system.  We describe our methodology in the following section.

%% file: sections/3_section.tex
\label{sec:methodology}
\subsection{Control Strategy}

We model the arrival and departure of robots 
by pure stochastic jump processes.
For example, for two tasks:
\begin{align}
\label{ex:two_tasks}
    X_1(t) \underset{\lambda_{\cdot}(\cdot)}{\longleftrightarrow}
    X_2(t)
\end{align}
the random variables $X_1(t)$ and $X_2(t)$ 
represent the number of robots in each task at time $t$,
and 
$\lambda_{\ell}(X_i)$ for
$\ell \in\{1,2\}$ maps the transition rate of agents between tasks. 
Motivated by Klavins \cite{klavins2010proportional} on gene regulatory networks,
we approach Problem $1$ by defining the transition rate functions
with a term that
is proportional with the number of agents
on a given site and in its neighborhood.
In equation \eqref{ex:two_tasks} this choice 
means that switching from
task $1$ to $2$ and from $2$ to $1$ will
depend on the number of agents at both 
tasks.
In particular, for equation \eqref{eq:inf_generator} 
we define the
transition rates as,
\begin{align}
    \label{eq:tran_rate_out}
    \lambda_1(X_i) = 
    &
    \sum_{j\in\NN_i}
    k_{ij}X_i-\beta_{i}X_iX_j,
    \\
    \label{eq:tran_rate_in}
    \lambda_2(X_i) = &
    \sum_{j\in\NN_i}
    k_{ji}X_j-\beta_{i}X_iX_j.
\end{align}

\begin{remark}
\label{remark:1}
A necessary condition for the equations in
\eqref{eq:tran_rate_out} and \eqref{eq:tran_rate_in} to be valid
transition rates 
is being non-negative. 
While it is possible to guarantee positiveness during
the whole evolution for some well-designed initial conditions with careful choices of
gains, we instead, let robots transition
back to previous tasks before the 
team reaches steady-state by mapping the equivalent non-negative transition with reversed direction.
This flexibility
is the justification for Assumption
\ref{assump:undirected}.
\end{remark}

It is worth noticing that
our control strategy is distinct from the one
presented in \cite{klavins2010proportional}. Here, we only
require the transition rate to be adjusted according to the current
number of agents in the tasks and in their neighborhood.
While in
\cite{klavins2010proportional} it is assumed that each node in the network
has {\it a  source and a sink}, which implies that the variation on the value of the
random variable $X_i$ can increase or decrease 
independently of the variation in 
its neighbors,
in addition to the switching between
adjacent nodes.

In the following, we show that with our choice for the transition rates
$\lambda_\ell(X_i)$,
the first-order moment depends only on 
the parameter
$k_{ij}$, while the variable $\beta_{i}$ manifests on
the second-order moments.
This introduces a free variable, 
allowing the variance to be properly
adjusted.

Defining $\psi(X_i)=X_i$ and 
applying the infinitesimal generator \eqref{eq:inf_generator} 
with the transition rates in \eqref{eq:tran_rate_out} and
\eqref{eq:tran_rate_in},
we get the following 
first-order moment for a general arrangement,
\begin{align*}
\frac{d}{dt}E[X_i] ={}
& E\Bigg[ 
\sum_{j\in \NN_i} 
\big( k_{ji}X_j-\beta_{i} X_iX_j\big)
-
\sum_{j\in \NN_i} 
\big( k_{ij}X_i-\beta_{i} X_jX_i\big)
\Bigg]
\nonumber
\\
=& {}  
\sum_{j\in \NN_i} 
\big(
k_{ji}E[X_j]
-k_{ij}E[X_i]
\big).
\end{align*}

While for the second-order moments
we define $\psi(X_i)=X_i^2$,
\begin{align*}
\frac{d}{dt}E[X_i^2] ={}
& E\Bigg[ 
\big((X_i+1)^2 - X_i^2 \big)
\sum_{j\in \NN_i} \big(k_{ji}X_j-\beta_{i}X_iX_j\big)
\\
&{} +
\big((X_i-1)^2 - X_i^2 \big)
\sum_{j\in \NN_i}
\big(k_{ij}X_i-\beta_{i}X_iX_j\big)
\Bigg]
\\
=&
~ E\Bigg[ 
\sum_{j\in \NN_i}
\big(2k_{ji}X_iX_j-2k_{ij}X_iX_i
+k_{ij}X_i
+k_{ji}X_j-2\beta_{i}X_iX_j\big) 
\Bigg],
\end{align*}
and for off-diagonal terms,
\begin{align*}
\frac{d}{dt}E[X_iX_q] ={}
& E\Bigg[ 
\big((X_i+1)X_q - X_iX_q \big)
\sum_{j\in \NN_i} \big(k_{ji}X_j-\beta_{i}X_iX_j\big)
\\
&{}
+\big((X_q+1)X_i - X_iX_q \big)
\sum_{j\in \NN_q} \big(k_{jq}X_j-\beta_{q}X_qX_j\big)
\\
&{} +
\big((X_i-1)X_q - X_iX_q \big)
\sum_{j\in \NN_i}
\big(k_{ij}X_i-\beta_{i}X_iX_j\big)
\\
&{}+
\big((X_q-1)X_i - X_iX_q \big)
\sum_{j\in \NN_q}
\big(k_{qj}X_q-\beta_{q}X_qX_j\big)
\Bigg]
\\
=&
~ E\Bigg[ 
\sum_{j\in \NN_i}
\big(k_{ji}X_q X_j -k_{ij}X_iX_q \big)
+
\sum_{j\in \NN_q} 
\big(k_{jq}X_iX_j-k_{qj}X_qX_i\big)
\Bigg].
\end{align*}
An important property of these equations
is that they are in
closed-form, which
allows us to use them to design gains
to attain the desired distribution.
For a compact vectorized form,
let ${\boldsymbol{K}}$ be an $M \times M$ matrix defined as follows,
\begin{align}
\label{eq:k_structure}
  \boldsymbol{K}_{ij} =
  \begin{cases}
    k_{ij}         & \text{if }(j,i)\in\EE,\\
    -\sum_q^M k_{iq} & \text{if }i = j, \\
    0              & \text{otherwise},
  \end{cases}
\end{align}
then the first- and second-order moments can be written as,
\begin{align}
    \label{eq:first_order}
    \frac{d}{dt}E[\boldsymbol{X}]=&
    ~\boldsymbol{K}
    E[\boldsymbol{X}],
    \\
    \label{eq:second_order}
    \frac{d}{dt}E[\boldsymbol{XX}']=&
    ~\boldsymbol{K}E[\boldsymbol{XX}']
    +E[\boldsymbol{XX}']\boldsymbol{K}'
    \nonumber
    \\
    &+\sum_{i=1}^Me_ie_i'\otimes \Bigg(\sum_{j\in\NN_i}
    \big(k_{ij}E[X_i]
+k_{ji}E[X_j] -2\beta_i
    E[X_iX_j] \big) \Bigg),
\end{align}
where $e_i$, for $i=1,...,M$, are the canonical basis of
$\real^M$ and 
$\otimes$ denotes the Kronecker product.
Equations \eqref{eq:first_order} and
\eqref{eq:second_order} model the dynamics of 
the first- and second-order moments of a group
of $N$ robots executing $M$ tasks.
There are two important underlying assumptions
on this representation,
first we assume the timing for robots to leave and
arrive at a task follows a Poisson process. 
We believe this assumption is not too restrictive
since, as discussed above, robot level deviations from 
desired performance will
impact the scheduling, leading
to stochastic
jump processes.
The second assumption, and arguably the hardest to justify
for general systems, is that {\it we can actuate on the
system's transition rates}.
To be able to do this in our case,
notice the structure of 
$\lambda_\ell(\cdot)$ in 
\eqref{eq:tran_rate_out} and 
\eqref{eq:tran_rate_in},
we need to monitor the current 
number of agents at each task and their neighborhood
and then communicate this value 
to each agent at these sites 
to compute the switching time.
Dynamically defining the transition rates has
an intricate relationship with the allocated time and
microscopic deviations.
We expect that the feedback nature of our approach
accounts for microscopic deviations from the desired
performance in the long run. 

\subsection{Control Analysis}
\label{sec:control_analysis}
In this section we analyse the convergence of the mean
to the desired distribution and 
show a simple method to obtain the parameters
of the gain matrix ${\boldsymbol K}$.
\begin{theorem}
\label{theorem:1}
Let $N$ agents execute $M$ tasks organized according to an
undirected graph $\GG(\VV,\EE)$.
Given a desired stationary distribution
$E[{\boldsymbol X}^d]$ such that $\sum_{i=1}^MX^d_i=N$,
the robot ensemble converges asymptotically to
$E[{\boldsymbol X}^d]$ from any initial distribution,
with transition rates
defined in \eqref{eq:tran_rate_out} and \eqref{eq:tran_rate_in},
if, and only if ${\boldsymbol K}E[{\boldsymbol X}^d]=0$.
\end{theorem}
\begin{proof}
The demonstration follows from standard Lyapunov analysis. 
Note that, by definition,
the eigenvalues of matrix ${\boldsymbol K}$ have
negative real part, except for one which is zero, that is,
$0=\sigma_1({\boldsymbol K})>\sigma_2({\boldsymbol K})\geq\cdots\geq\sigma_M({\boldsymbol K})$, where
$\sigma_i({\boldsymbol K})$ denotes the $i$th eigenvalue of ${\bf K}$.
Hence, defining a simple Lyapunov function candidate,
    $V = E[{\boldsymbol X}]' P E[{\boldsymbol X}],$
where $P$ is a symmetric positive definite matrix.
The time-derivative of $V$ along the trajectories of the first-order moments yields
\begin{align*}
    \frac{dV}{dt} =
    E[{\boldsymbol X}]' \big( {\boldsymbol K}' P  
    +  P {\boldsymbol K} \big) E[{\boldsymbol X}],
\end{align*}
because every eigenvalue of ${\boldsymbol K}$ has
negative or zero real part, we have that
\begin{align*}
    \frac{dV}{dt} \leq 0.
\end{align*}
Consider the invariant set $S=\{E[{\boldsymbol X}]\in\real^M:\frac{dV}{dt}=0\}$,
since the robots will always be in one of the graph's nodes, {\it i.e.,}
$\sum_{i=1}^MX_i=N$, we have that $E[{\boldsymbol X}(t)]\neq \boldsymbol{0}$ for all $t>0$.
Therefore, no solution can stay in $S$ other than $E[{\boldsymbol X}^d]$, in other words,
the equilibrium point $E[{\boldsymbol X}^d]$ is globally asymptotically stable.
\end{proof}

Notice
that we assumed
that each
robot will always
be assigned to
one task at any
instant, to
consider
robots in transit between nodes 
it is possible to
extend the task 
graph as in \cite{berman2009}.
In addition,
by construction, ${\boldsymbol K}$ results in a stable ensemble.
The compelling aspect of Theorem \ref{theorem:1} with respect to our model
for the first- and second-order moment, equations
\eqref{eq:first_order} and \eqref{eq:second_order},
is that we can use the condition provided on the theorem to design the
gains to attain a desired distribution, while still having free
variables $\beta_i$ to adjust the second-order terms.
We compute the gain matrix ${\boldsymbol K}$ by solving the following
optimization:
\begin{align}
\label{eq:optimizationK}
    {\boldsymbol K} = {}&\underset{{\boldsymbol Q}}{\rm{argmin}}
    \lVert {\boldsymbol Q } E[{\boldsymbol X}^d] \lVert
    \\
    \nonumber
    {}&\rm{s.t.} ~{\bf 1}'{\boldsymbol Q}=\boldsymbol{0};~{\boldsymbol Q}\in \mathcal{K};
    ~\text{structure in \eqref{eq:k_structure}},
\end{align}
where the set $\mathcal{K}$ accounts for the set of possible
rate matrices for a given set of tasks and 
$E[{\boldsymbol X}^d]\in{\rm Null}({\boldsymbol Q}) $ is the desired final distribution.

At the moment,
although we have shown
additional tuning variables in our approach,
we do not have
a methodology to compute appropriate gains $\beta_i$ for
a desired covariance matrix.
Nonetheless, we give an intuition of the
mechanism of how they impact the steady-state
covariance matrix. Let 
${\boldsymbol C}=E[({\boldsymbol X} -E[{\boldsymbol X}])({\boldsymbol X} -E[{\boldsymbol X}])']$ be the covariance matrix, 
replacing $E[{\boldsymbol X}{\boldsymbol X}']$ in  \eqref{eq:second_order} and recalling that 
${\boldsymbol K}E[{\boldsymbol X}]=\boldsymbol{0}$ during the steady-state, 
analysing the equilibrium 
point $\frac{d}{dt}E[\boldsymbol{XX}']=\boldsymbol{0}$ gives us
\begin{align*}
    {\boldsymbol{KC}} +  {\boldsymbol C}{\boldsymbol K}'
    = 
    &
    -\sum_{i=1}^Me_ie_i'\otimes \Bigg(\sum_{j\in\NN_i}
    \big(k_{ij}E[X_i^d]
+k_{ji}E[X_j^d] -2\beta_i
    E[X_i^dX_j^d] \big) \Bigg).
\end{align*}
Since every eigenvalue of ${\boldsymbol K}$ has negative real part except for one which is equal to zero,
and $E[{\bf X}^d] \neq \boldsymbol{0}$, we have an unique solution for ${\boldsymbol C}$ \cite{KHALIL2002ab,klavins2010proportional}.
Therefore, the
additional tuning variables 
$\beta_i$, for $ i\in\VV$
are inversely proportional to
the covariance at the 
steady-state--the bigger the values of $\beta_i$, while \eqref{eq:tran_rate_out} and \eqref{eq:tran_rate_in}
are positive during the steady-state 
(recalling Remark \ref{remark:1}),
the smaller the covariance.

%% file: sections/5_simulation.tex
\begin{figure}[ht]
    \centering
    \begin{subfigure}[t]{\textwidth}
    \centering
    \includegraphics[scale=0.21]{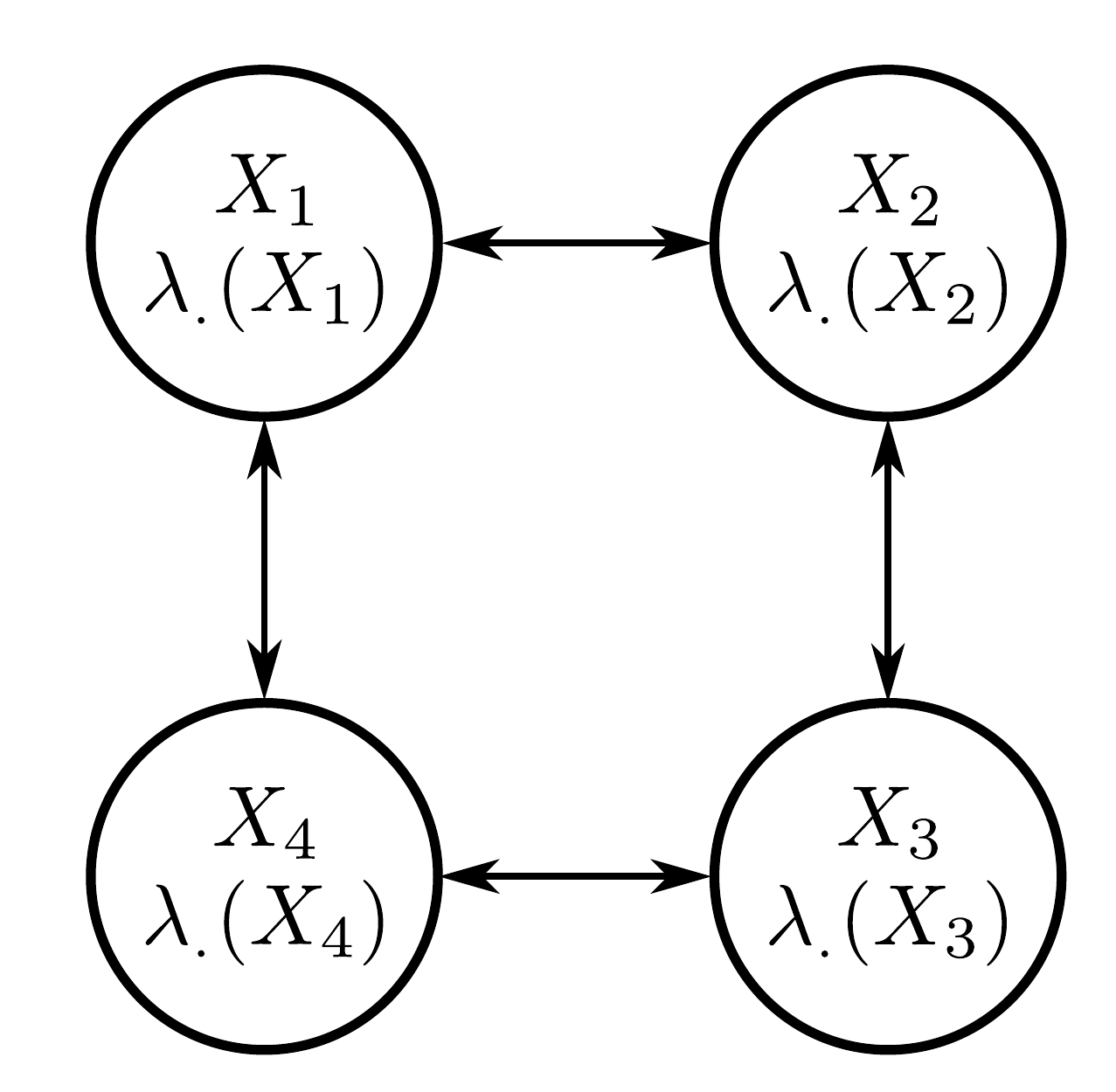}
    \caption{Graph Topology}
    \label{fig:four_task_graph}
    \end{subfigure}
    \\
    \begin{subfigure}[t]{0.47\textwidth}
    \centering
    \includegraphics[width=\textwidth]{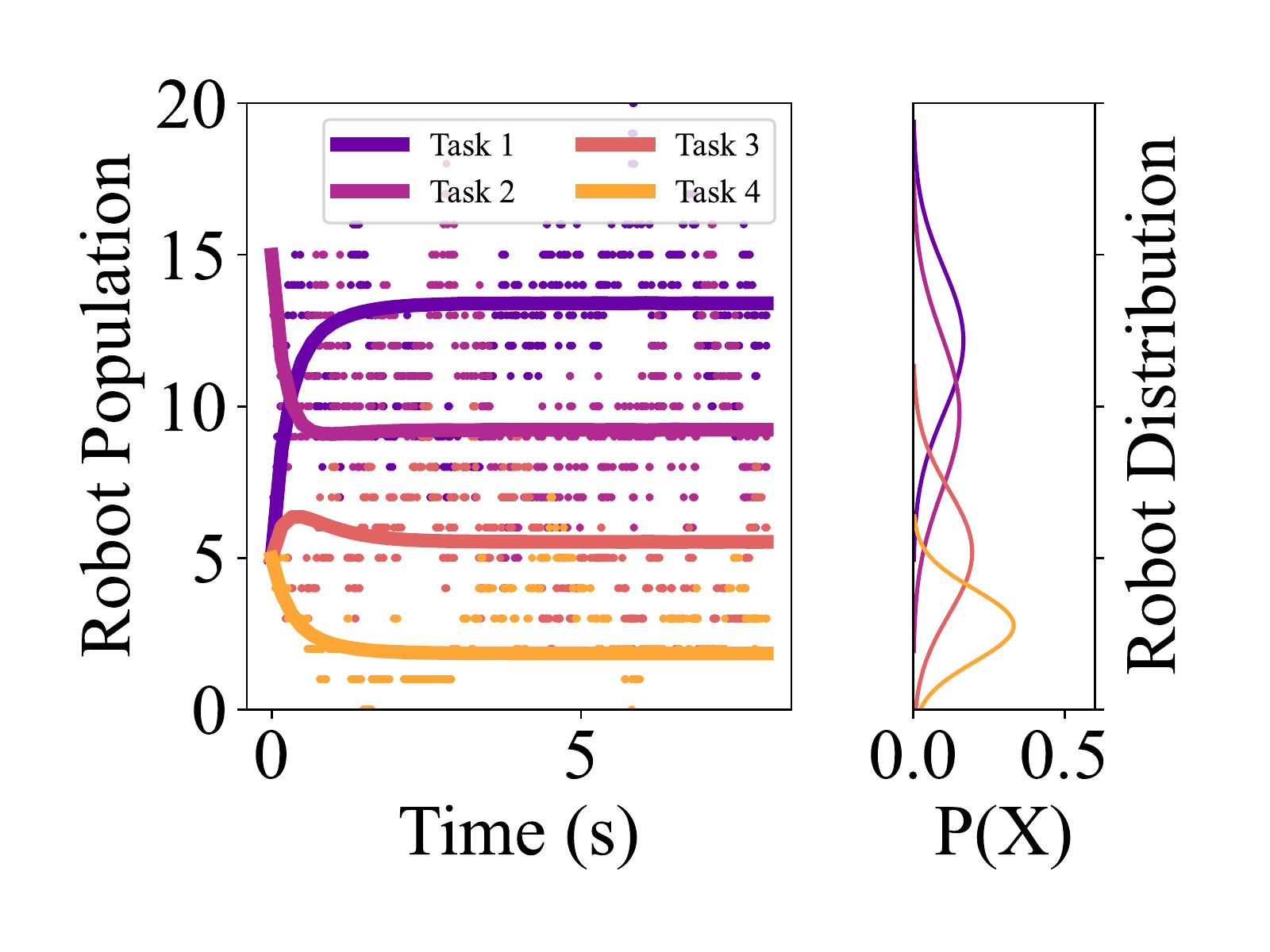}
    \caption{Realization with 
    $\beta_i=0,~\forall i\in \VV$}
    \label{fig:four_task_without_control}
    \end{subfigure} 
    \begin{subfigure}[t]{0.47\textwidth}
    \centering
    \includegraphics[width=\textwidth]{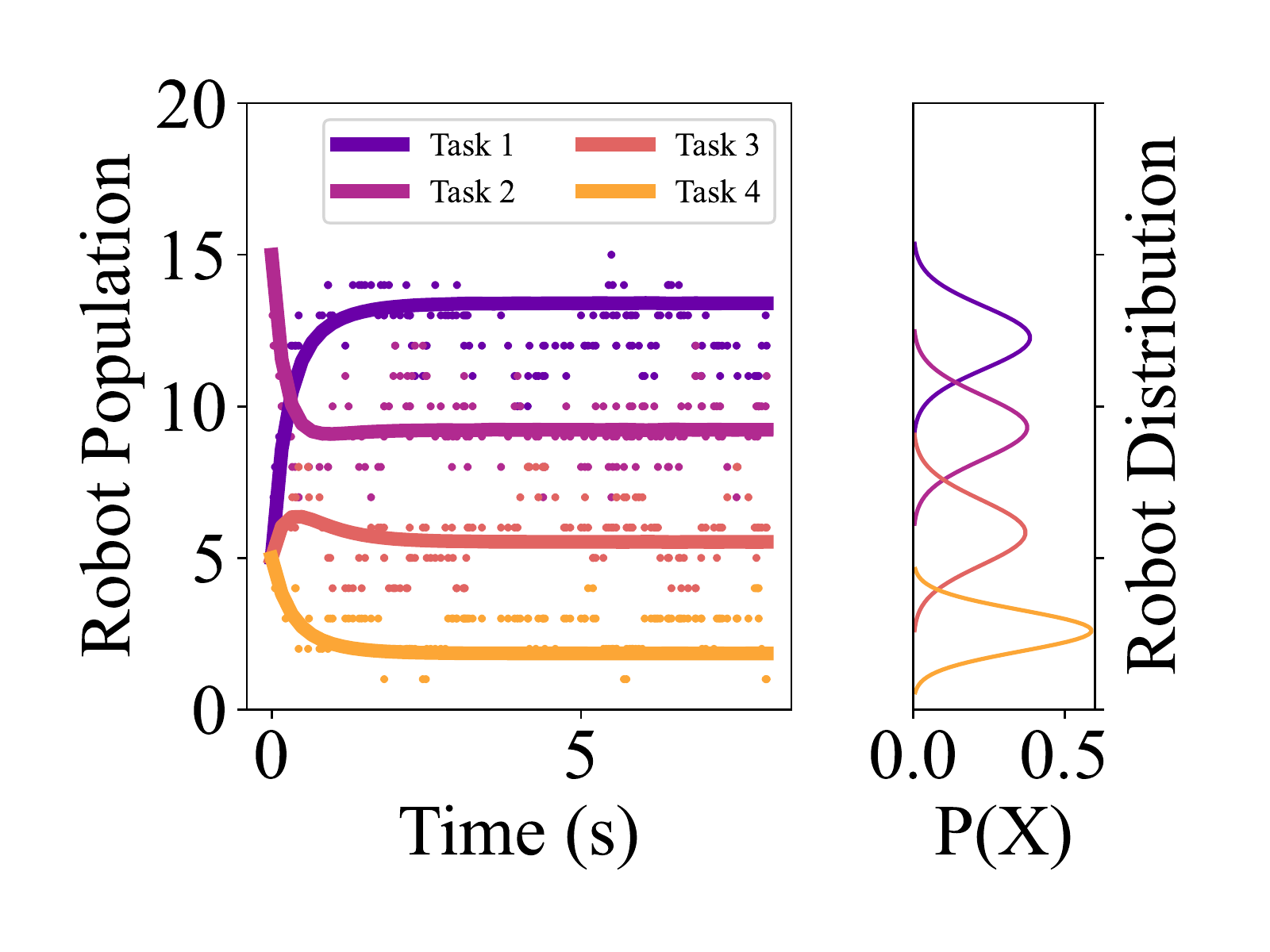}
    \caption{Realization with 
    $\beta_i\neq 0,\forall i\in \VV$}
    \label{fig:four_task_with_control}
    \end{subfigure}
    \caption{
    The graph topology is depicted in (a).
    In (b) and (c) it is shown
    a realization of each of the
    stochastic simulations, where the dots
    represent the number of agents 
    and the solid lines
    represent the numerical solution of
    \eqref{eq:first_order}. 
    The tasks are identified by
    the color caption.
    The realization in (b) uses
    $\beta_i= 0,~\forall i\in \VV$,
    while in (c) is with
    $\beta_1 = 0.05$, $\beta_2 = 0.20$, $\beta_3 = 0.11$, and $\beta_4 = 0.052$.
    } 
    \label{fig:four_task_example}
\end{figure}

\subsection{\bf Numerical Simulations}
\begin{example}
We use the stochastic simulation algorithm
proposed by Gillespie \cite{gillespie1977exact}
to evaluate the proposed control strategy.
We present a four task example, $M = 4$,
with the topology
represented in Figure \ref{fig:four_task_example} (a).
We run two different tests 
to evaluate the effectiveness
of our method, one with parameters
$\beta_i=0$, for $i=1,\dots,4$, and another
with nonzero  $\beta_i$ parameters (values given
below).
In both
cases, we aim to achieve the same 
final distribution.
The initial populations at each
task is given by 
${\boldsymbol X}_0 = [ 5 ~ 15 ~ 5 ~ 5]'$, and
the desired distribution is 
$E[{\boldsymbol X}^d] = [ 13 ~ 9 ~ 6 ~ 2 ]'$.
The gain matrix 
${\boldsymbol K}$ used to reach such a 
distribution was
computed solving the optimization
problem \eqref{eq:optimizationK} 
with $\mathcal{K}=\{{\boldsymbol K}\in\real^{4\times4}
:{\boldsymbol K}\leq -1.5I \}$, this
constraint was imposed to ensure 
not too long convergence time. 
The resulting parameters from the
optimization are $k_{12} = 2.1$,
$k_{14} = 1.4$,
$k_{21} = 1.5$,
$k_{23} = 1.3$,
$k_{32} = 0.9$,
$k_{34} = 1.2$,
$k_{41} = 0.1$, 
$k_{43} = 0.6$, and $0$ elsewhere outside
the main-diagonal. 
We computed the mean and variance from
the simulation considering
$130$ data points sampled for $t > 2.0$ seconds in each simulation.
This number was chosen to match the sample size between the two
experiments.
The simulation with 
$\beta_i = 0$ for all $i=1,\dots,4$ 
in 
Figure \ref{fig:four_task_example} (b)
has the following mean 
$E[{\boldsymbol X}] = 
[ 12.16 ~ 9.76 ~ 5.18 ~ 2.77]'$,
and variances,
$\diag(E[{\boldsymbol{XX}}']) = [ 5.78 ~ 6.83 ~ 4.20 ~ 1.44]'$.
The variances were improved by greedily modifying 
$\beta_{i}$ values--if the variance was reduced and
did not actively disturb another
task population variance then that
value was kept. 
Figure \ref{fig:four_task_example} (c) used parameters
    $\beta_1 = 0.05$, $\beta_2 = 0.20$, $\beta_3 = 0.11$, and $\beta_4 = 0.052$,
where the resulting means are 
$E[{\boldsymbol X}] = [ 12.26 ~ 9.30 ~ 5.83 ~ 2.60 ]'$ 
variances are 
$\diag(E[{\boldsymbol{XX}}']) = [1.06 ~ 1.12 ~ 1.15 ~ 0.45 ]'$. 

These results highlight the impact of variability on small
size teams ($N=30$).
Visually, the control reduces the variance throughout the
experiment, especially once steady-state is reached.
It is worth mentioning that we have tested our strategy on larger
teams and the results were in agreement with our hypothesis
that it is possible to shape the variance using
\eqref{eq:tran_rate_out} and \eqref{eq:tran_rate_in}, however,
due to space limitations of the venue, we chose to not show
those results.
\end{example}

\begin{example}
In this example, we run $6$ different trials with
varying team sizes to
numerically evaluate the variation in accuracy for teams of
different sizes. To this end, we consider four tasks
with topology represented in Figure \ref{fig:four_task_example} (a),
and we computed the variance considering
$160$ 
data points sampled from each simulation for
$t>2.0 $ seconds.
In addition, we vary the total number of robots
between simulations, hence relative the initial
and desired distributions for each simulation are:
${\boldsymbol X}_0 = [ 25\% ~ 25\% ~ 0\% ~ 50\%]'$, and
$E[{\boldsymbol X}^d] = [ 50\% ~ 50\% ~ 0\% ~ 0\% ]'$,
respectively,
given
as a
percentage of the total number of robots in the team. 
In these simulations, we
used the same
values for 
the parameters
$\beta_i$
as in Example 1,
and we do not change them between simulations.
To provide an intuition of the level of spread in each
scenario, we compute the Relative Variance (RV) as
the quotient of the mean by the variance of the
respective task. The results are presented in
Table \ref{tab:table1}.
\begin{table}[ht]
\begin{center}
\begin{tabular}{ c|c|c|c||c|c|c  }
 \hline
 \multicolumn{4}{c||}{$\beta_i=0$ for $i=1,\dots,4~$}
 &\multicolumn{3}{|c}{$~\beta_i\neq0$ for $i=1,\dots,4~$} \\
 \hline
 & $N=52$ & $~N=26~$& $~N=16~$ &$~N=52~$&$~N=26~$&$~N=16~$\\
 \hline
 $E[\boldsymbol{X}_1]$&$24.8$ & $12.8$ & $8.2$  &$25.2$ & $13.3 $& $7.5 $\\
 $E[\boldsymbol{X}_2]$&$26.5$ & $12.7$ & $7.7$  &$25.6$ & $14.6 $& $8.3 $\\
 $E[\boldsymbol{X}_3]$&$0.5$  & $0.0$  & $0.0$  &$0.6$  & $0.0 $ & $0.1 $ \\
 $E[\boldsymbol{X}_4]$&$0.2$  & $0.0$  & $0.0$  &$0.4$  & $0.0 $ & $0.0 $ \\
 \hline
 RV($\boldsymbol{X}_1$)&$0.49$& $0.62$ & $0.61$  &$0.21 $ & $0.36 $& $0.53 $\\
 RV($\boldsymbol{X}_2$)&$0.54$& $0.63$ & $0.69$  &$0.21 $ & $0.23 $& $0.47 $\\
 RV($\boldsymbol{X}_3$)&$0.97$& $0.1$ & $0.0$  &$0.88 $ & $0.00 $& $1.56 $\\
 RV($\boldsymbol{X}_4$)&$0.75$& $0.1$ & $0.0$  &$1.38 $ & $0.00 $& $0.00 $\\
 \hline
\end{tabular}
\bigskip
\caption{\label{tab:table1}
Expectation and Relative Variance (RV) 
of a
series of
numerical experiments considering different team sizes, with
task graph in Figure \ref{fig:four_task_example} (a),
initial and desired final distributions 
${\boldsymbol X}_0 = [ 25\% ~ 25\% ~ 0\% ~ 50\%]'$, and
$E[{\boldsymbol X}^d] = [ 50\% ~ 50\% ~ 0\% ~ 0\% ]'$,
respectively, and when $\beta_i$ parameters
are considered their values are given by
    $\beta_1 = 0.05$, $\beta_2 = 0.20$, $\beta_3 = 0.11$, and $\beta_4 = 0.052$.
}
\end{center}
\end{table}
We notice that for this scenario, the RV
for the multi-robot system without $\beta_i$ and
with $52$ robots is similar to the RV of a team
with only $16$ robots and $\beta_i$ values.
This suggests that our methodology improves the accuracy
of the allocation of relatively small teams.
Figures \ref{fig:four_task_example2} (a) and (b) provide a realization for an instance with 
$\beta_i=0$ and $N=52$ and with $\beta_i\neq0$ and $N=16$.
\begin{figure}[ht]
    \centering
    \begin{subfigure}[t]{0.47\textwidth}
    \centering
    \includegraphics[width=\textwidth]{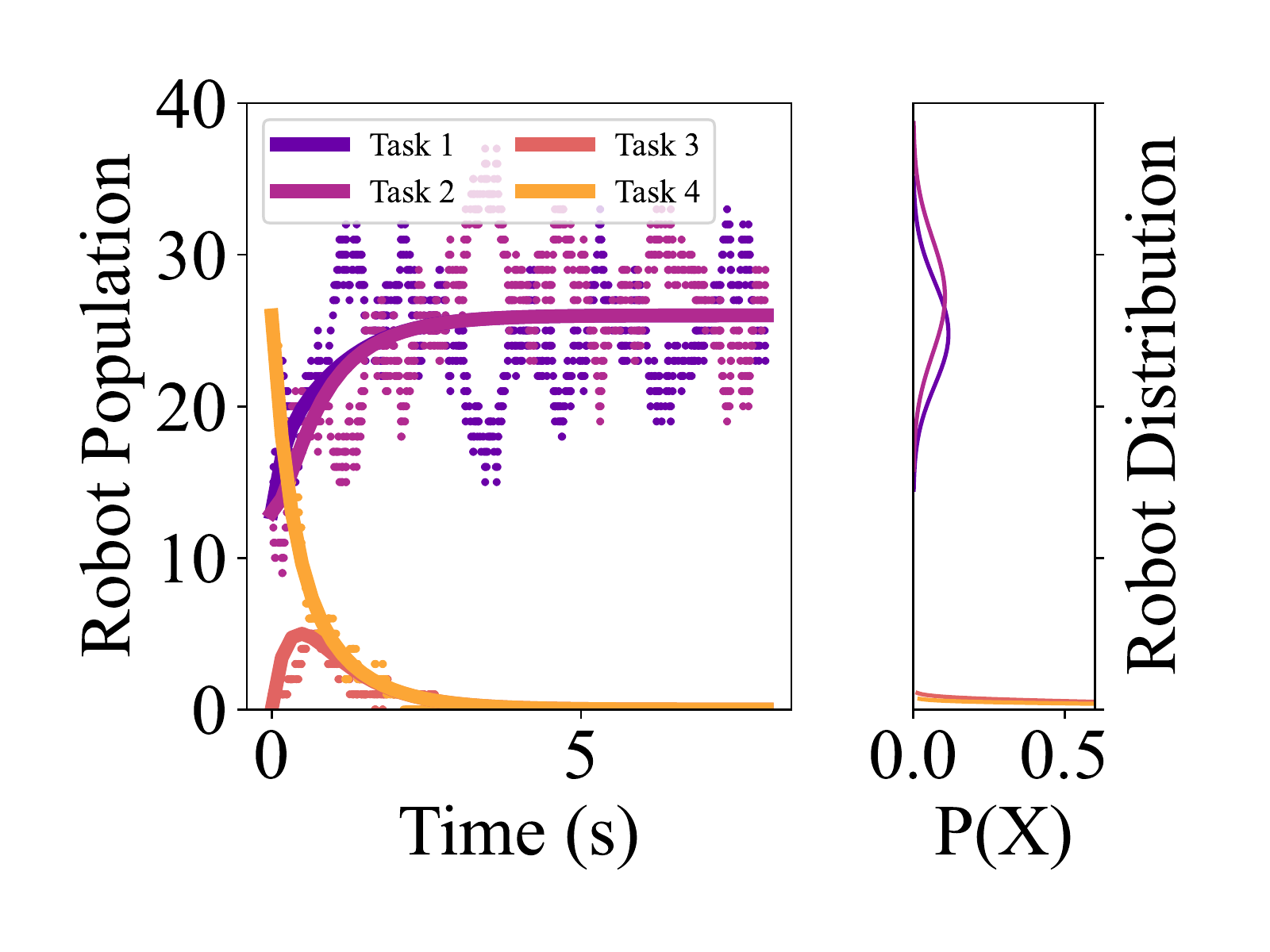}
    \caption{Realization without 
    $\beta_i$ and $N=52$}
    \end{subfigure} 
    \begin{subfigure}[t]{0.47\textwidth}
    \centering
    \includegraphics[width=\textwidth]{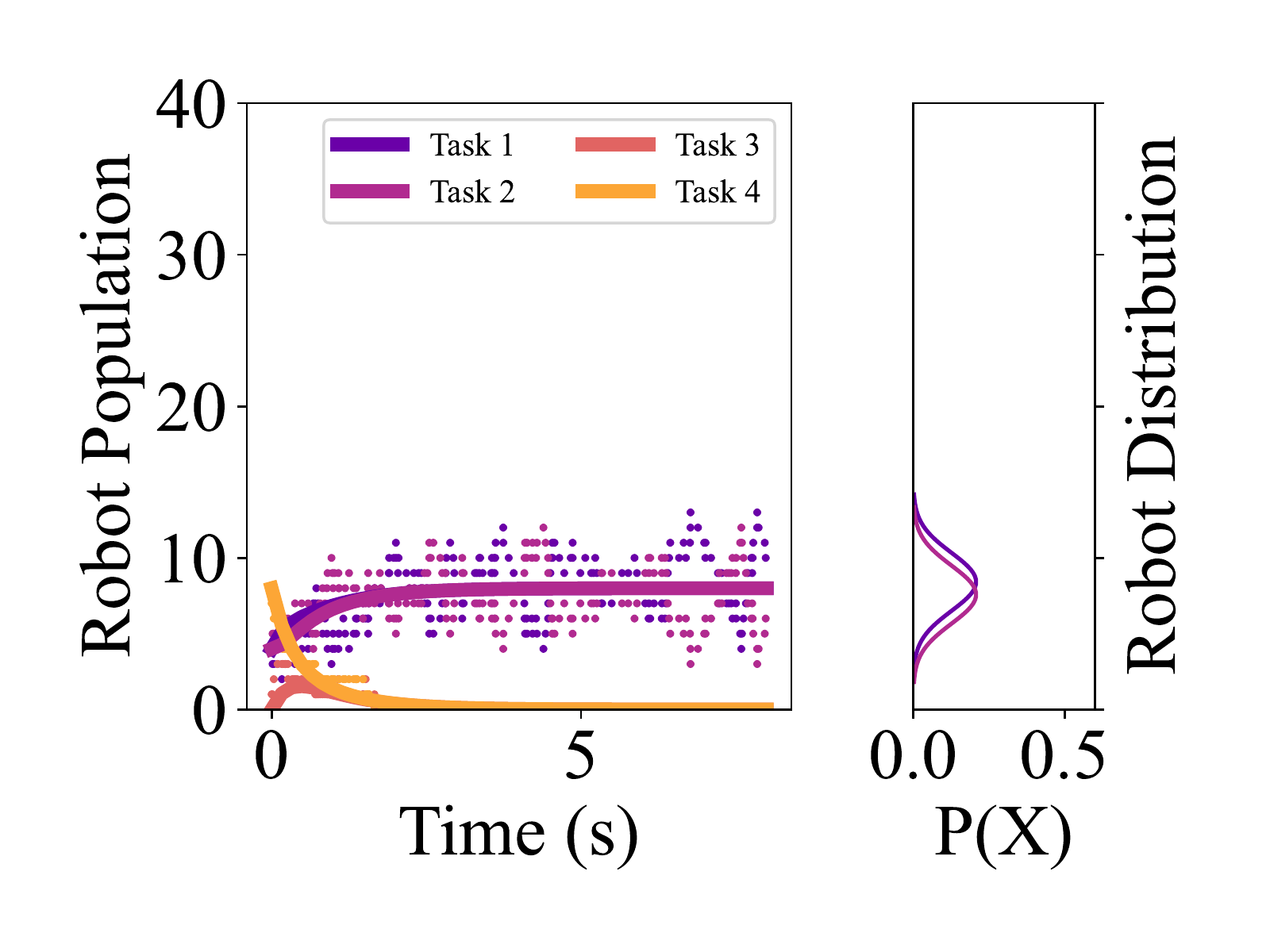}
    \caption{Realization with 
    $\beta_i$ and $N=16$}
    \end{subfigure}
    \caption{
    In (b) and (c) it is shown
    a realization of each of the
    stochastic simulations with $N=52$ and $N=16$,
    respectively, where the dots
    represent the number of agents 
    and the solid lines
    represent the numerical solution of
    \eqref{eq:first_order}. 
    The figure axis are the same to help the visualization
    of the relative spread.
    The tasks are identified by
    the color caption
    and  
    the graph topology is depicted in Figure \ref{fig:four_task_example} (a).
    The realization in (b) uses
    $\beta_i= 0,~\forall i\in \VV$,
    while in (c) is with
    $\beta_1 = 0.05$, $\beta_2 = 0.20$, $\beta_3 = 0.11$, and $\beta_4 = 0.052$.
    } 
    \label{fig:four_task_example2}
\end{figure}
\end{example}

\begin{figure}[ht]
    \centering
    \begin{subfigure}{0.4\textwidth}
    \centering
    \includegraphics[scale=0.035]{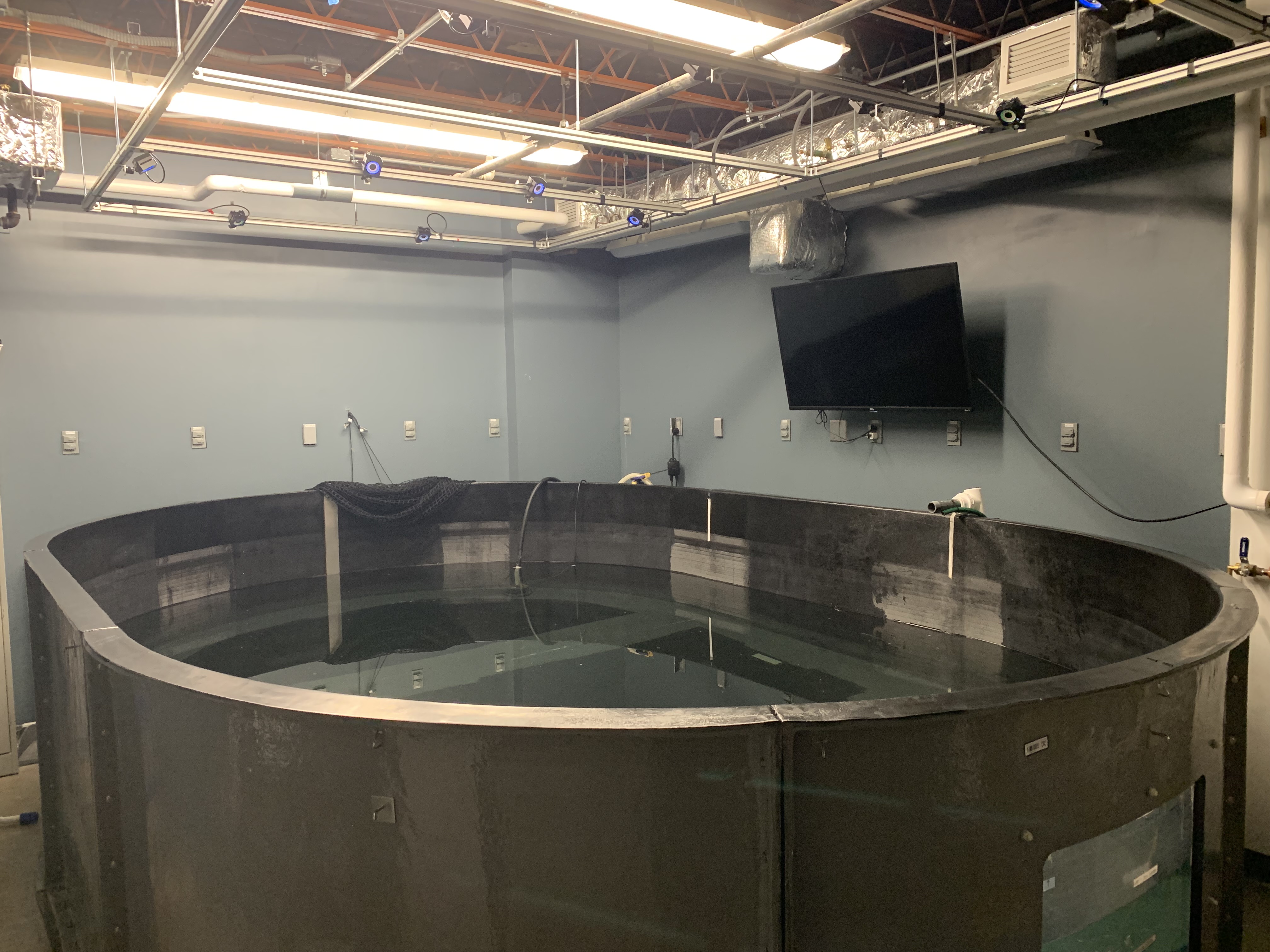}
    \caption{mCoSTe Environment}
    \label{fig:exp_setup}
    \end{subfigure} 
    \begin{subfigure}{0.45\textwidth}
    \centering
    \includegraphics[scale=0.1103]{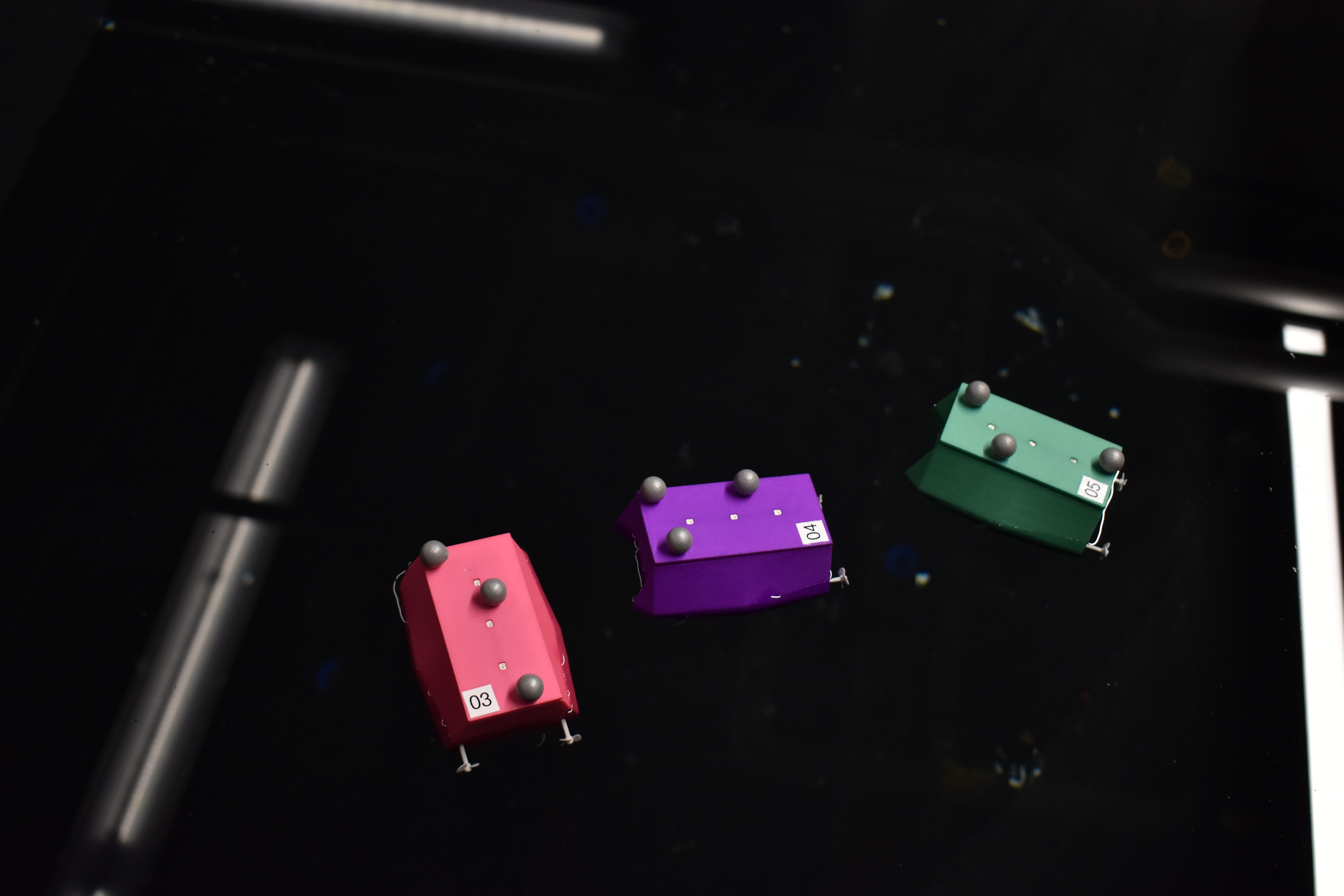}
    \caption{mASV Platform}
    \label{fig:tiny_boats}
    \end{subfigure}
    \caption{ The mCoSTe environment and mASV platform used for experimental results.} 
    \label{fig:exp_env}
\end{figure}
\begin{figure}[ht]
    \centering
    \begin{subfigure}{0.48\textwidth}
    \centering
    \includegraphics[scale=0.23]{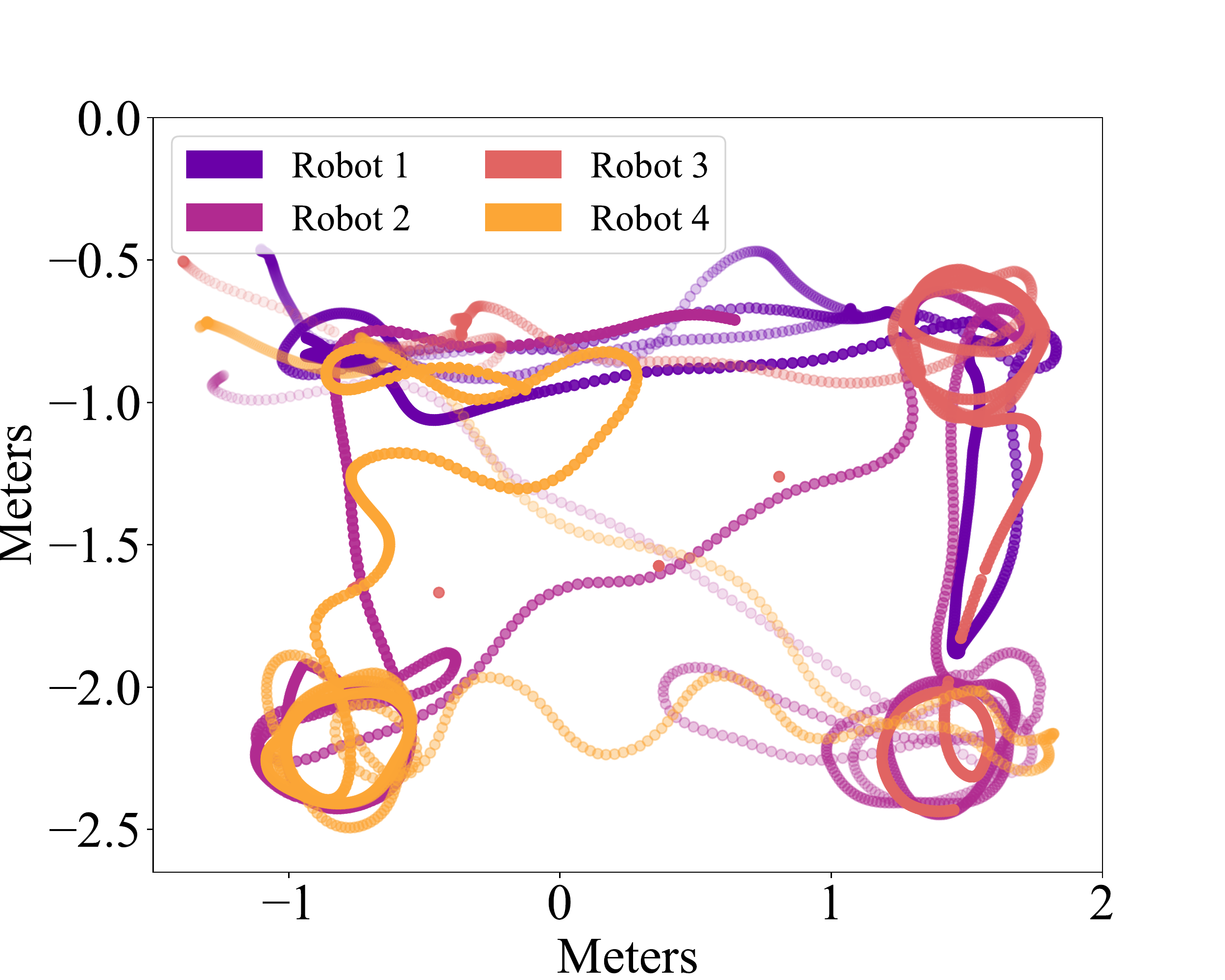}
    \caption{Trajectories with $\beta_i=0$}
    \label{fig:exp_open_loop}
    \end{subfigure} 
    \begin{subfigure}{0.48\textwidth}
    \centering
    \includegraphics[scale=0.23]{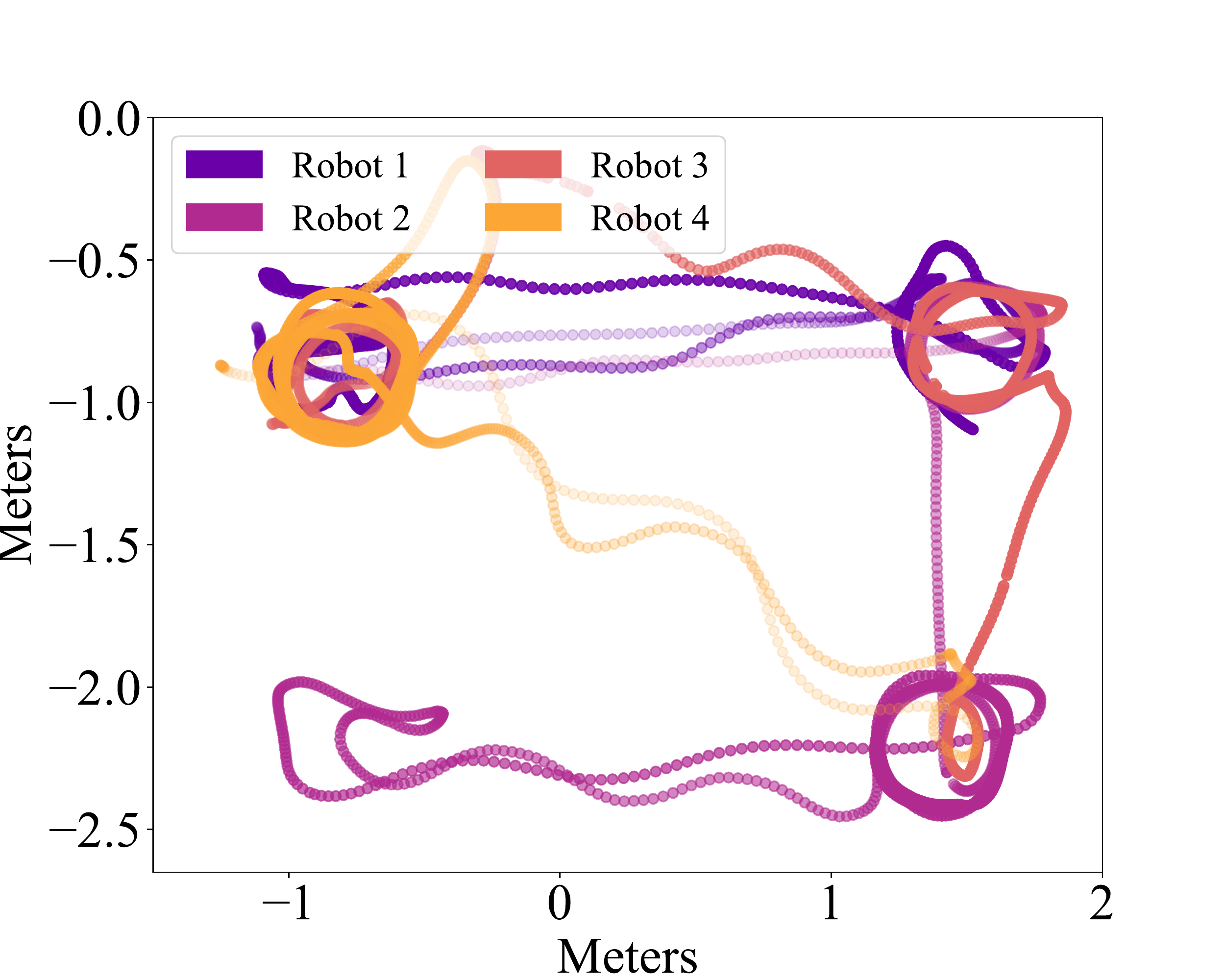}
    \caption{Trajectories with $\beta_i\neq0$}
    \label{fig:exp_closed_loop}
    \end{subfigure}
    \caption{
    Experimental results from a $4$ minute and $4$ mASV
    run, with task graph in 
    Figure \ref{fig:four_task_example} (a). 
    The desired distribution is one agent at each task.
    The test in Figure \ref{fig:exp_figure} (a) had 28 switches, 
    while in Figure 
    \ref{fig:exp_figure} (b) it had 13 switches.
    The colors' intensity reflects time--the lighter the color
    the earlier that path was cruised in the trial.
    } 
    \label{fig:exp_figure}
\end{figure}

\subsection{Experimental Results}
Experimental tests were performed in the multi-robot Coherent Structure Testbed (mCoSTe).
Figure \ref{fig:exp_env} depicts the miniature autono-mous surface vehicles (mASV) in
Figure \ref{fig:exp_env} (b),
and the Multi-Robot Tank (MR tank) in Figure \ref{fig:exp_env} (a).
The MR Tank is a $4.5 \mathrm{m} \times 3.0 \mathrm{m} \times 1.2 \mathrm{m}$
water tank equipped with an OptiTrack motion capture system. 
The mASVs are a differential drive platform, localized
using $120$ Hz motion capture data, communication link 
via XBee, and onboard processing with an Arduino Fio.
Experiments were run for $4$ minutes each, with 
${\boldsymbol X}_0 = [4 ~ 0 ~ 0 ~ 0]'$, and 
$E[{\boldsymbol X}^d] = [ 1 ~ 1 ~ 1 ~ 1]'$,
following the task graph outlined in Figure \ref{fig:four_task_example} (a). 
The parameters used for experimental trials were $k_{ij} = 0.01$ for each nonzero and out-of-diagonal element of
the matrix ${\boldsymbol K}$,
and when variance control was used it was all $\beta_{i} = 0.005$
for $i=1,\dots,4$.
Those parameters were chosen to take into consideration the
time necessary for the mASVs to travel among tasks.
The parameters were converted to individual robot transitions 
by computing the likelihood of transition using equations \eqref{eq:tran_rate_out} and \eqref{eq:tran_rate_in}, as 
well as which task the robot should transition to. 
At each task the mASVs were required
to circle the location with a radius of 0.25m
until a new transition was required.

An instance of an experimental trial for both, 
with and without $\beta_i$, are included in Figure \ref{fig:exp_figure}.
The trajectory of each robot
is represented by colors,
the intensity of the colors is associated with its time
occurrence--when the color is lighter that path was traveled
by the respective mASV
earlier in the experimental trial. 
Figure \ref{fig:exp_figure} (a) has a total of $28$
transitions among tasks, those transitions take place
throughout the experiment as indicated by darker lines between tasks. 
In the case of variance control,
Figure \ref{fig:exp_figure} (b),
there are $13$ 
transitions among tasks, 
notice that many of these transitions happened 
earlier in the experiment. 
From five experimental trials without variance control the average number of task switches was $25.4$,
and from five experimental trials with variance control
the average number of task switches was $15.6$.
This confirms that in the case where feedback is provided there is a reduced number of switches. 
A video illustrating our results can be found at \url{https://youtu.be/fTsX-5z6BUw}.

During experimentation we observed that sometimes the mASVs would
transition before it could physically reach  
the assigned task,
traveling diagonally
among tasks. 
This was the case in Figure \ref{fig:exp_figure} (b)
where Robot $4$ 
quickly switched to another task
while
in transit. 
While this did not impact the overall results, it is an area of
open interest to achieve desired parameters which
leads to the desired task distribution within the capability of the robots.

%% file: sections/discussion.tex
In the preceding
sections, we have investigated
the problem of stochastic allocation
of multi-robot teams.
For our main result, we formally
demonstrated that through a particular
structure for the transition rates
of the stochastic jump process model,
we can decouple the first-order moments
from the second-order moments.
Such a decoupling allows us to
introduce additional tuning variables
to regulate the variance of the
desired distribution of robots over
tasks. 
The additional degree of freedom
helps to reduce the impact of individual
robots on the overall team's performance.
Therefore, the
result of this contribution is to expand the viability of top-down stochastic models to include reduced-size robot teams.
In general, the intuition for these models is that they are more accurate as the team size increases.
However, when using our proposed method and directly adjusting the team variance it is 
possible
to increase the top-down stochastic model accuracy for smaller teams.
We argue that such refinement is a stride
towards sparse swarms--even though we are technically
approaching the team size problem and ignoring the
task space size.

Although we have
formally shown the impact of the
additional tuning variables 
as well as
the decoupling between first- and 
second-order moments,
and also
experimentally and numerically 
investigated the influence of
such variables, 
 we did not draw general mathematical expressions to compute their
 values during the design stage.
 In our investigations, we used a
 greedy algorithm that increased the
 variable number based on the desired 
 values of
 variance and 
 final
 distribution.
 In addition, our optimization
 problem to compute the gain matrix
 $\boldsymbol{K}$ for 
 a desired distribution incorporates
 designer knowledge about the 
 robot capabilities through
 the set $\mathcal{K}$, which
 will directly affect the
 robots' mean transition time.
 We plan to formalize this unsolved
 concern in future investigations.

%% file: sections/6_conclusion.tex
We have provided a new structure for the transition rates for 
ensemble robot distributions modeled from a top-down perspective.
As a first result, 
we have demonstrated that such a structure 
leads to uncoupled parameters to adjust the mean and the variance of
desired team's distribution.
Then, based on this finding, we examined simple design strategies
to compute the necessary gains.
This approach provides an efficient ensemble behavior
for relatively small groups.
Finally, physical and numerical experiments were implemented, illustrating
the effectiveness of the method.
Possible future work includes the extension of the strategy for
distributed regulation.
A potential strategy is to perform
distributed estimates of the number of agents
performing each task.
It is also of interest to connect the robot dynamics with the changing 
transition rates. 
One possible approach to bridging those two models is through 
hybrid switching systems.
A formal methodology to design the network structure for a given covariance bound will be considered
in the future.